\documentclass[letterpaper]{article}
\usepackage{kr}

\usepackage{times}
\usepackage{soul}
\usepackage{url}
\usepackage[hidelinks]{hyperref}
\usepackage[small]{caption}
\usepackage{booktabs}
\usepackage{graphicx}

\newcommand{\tcite}[1]{\cite{#1}}
\newcommand{\pcite}[1]{\cite{#1}}

\usepackage{color}
\usepackage{subcaption}
\usepackage[utf8]{inputenc}
\usepackage[a4paper,margin=2cm]{geometry}
\usepackage{fancyhdr}
\usepackage{amsmath}
\usepackage{amsthm}
\usepackage{amssymb}
\usepackage{amsfonts}
\usepackage{qtree}
\usepackage[counterclockwise]{rotating}
\usepackage{scalerel}
\usepackage{textgreek}
\usepackage{algorithm}
\usepackage{algpseudocode}
\usepackage{tikz,forest}
\usepackage{tablefootnote}
\usepackage{dirtytalk}
\usetikzlibrary{arrows.meta}
\usepackage{xspace}
\usepackage{wrapfig}

\usepackage{dcolumn}
\newcolumntype{.}{D{.}{.}{-1}}
\makeatletter
\newcolumntype{B}{>{\boldmath\DC@{.}{.}{-1}}c<{\DC@end}}
\newcolumntype{E}{>{\centering\DC@{.}{-}{-1}}c<{\DC@end}}
\makeatother
\newcommand\mc[1]{\multicolumn{1}{c}{#1}}
\newcommand\bft[1]{\multicolumn{1}{B}{#1}}
\newcommand\emptyt{\multicolumn{1}{E}{\quad -}}

\tikzset{Variable/.style={circle,thick,draw}}
\tikzset{Factor/.style={rectangle,thick,draw}}
\tikzset{WeirdArrow/.style={dashed}}
\tikzset{every edge/.style={draw=black,very thick}}

\theoremstyle{definition}

\newtheorem{defff}{Definition}

\newenvironment{deff}[1]{%
    \begin{defff}#1}{
    \end{defff}%
}

\theoremstyle{definition}
\newtheorem{exmp}{Example}

\newtheorem{prop}{Proposition}

\newcommand{\predicates}{\mathcal{P}}

\newcommand{\fol}{\mathcal{L}}

\newcommand{\objects}{O}

\newcommand{\btheta}{{\boldsymbol{\theta}}}
\newcommand{\luk}{\L{}ukasiewicz}
\newcommand{\loss}{\mathcal{L}}
\newcommand{\corpus}{\mathcal{K}}

\newcommand{\interpretationfol}{\eta}
\newcommand{\interpretation}{\eta_{\btheta}}

\newcommand{\val}{e_{\btheta}}

\newcommand{\fullval}{e_{\interpretation, p, T, S, I, A}}

\newcommand{\deri}{d}
\newcommand{\dmp}{\deri_{Ic}}
\newcommand{\dmpa}[1]{\deri_{#1c}}
\newcommand{\dmt}{\deri_{I\neg a}}
\newcommand{\dmta}[1]{\deri_{#1\neg a}}
\newcommand{\instantiations}{M}
\newcommand{\instantiation}{\mu}

\newcommand{\bx}{{\boldsymbol{x}}}

\newcommand{\graphwidth}{0.45\linewidth}

\newcommand{\cons}{\mathrm{cons}}
\newcommand{\ant}{\mathrm{ant}}
\newcommand{\mpmag}{|\cons|}
\newcommand{\mtmag}{|\ant|}
\newcommand{\mpratio}{\cons\%}

\newcommand{\mpcorupdate}{{cu_{\cons}}}
\newcommand{\mtcorupdate}{{cu_{\ant}}}

\newcommand{\mpupdateratio}{cu_\cons\%}
\newcommand{\mtupdateratio}{cu_\ant\%}

\newcommand{\dfl}{DFL\xspace}

\newcommand{\dfuzz}{Differentiable Fuzzy Logics\xspace}

\newcommand{\predP}{\pred{P}}

\newcommand{\dualfigure}[4]{\begin{figure}[h]
\centering
\begin{subfigure}[b]{\graphwidth}
\includegraphics[width=\linewidth]{#1}
\end{subfigure}
\begin{subfigure}[b]{\graphwidth}
\includegraphics[width=\linewidth]{#2}
\end{subfigure}%
\caption{#3}
\label{#4}
\end{figure}}

\newcommand{\imgT}{images/theory_kr/}
\newcommand{\imgE}{images/experiments/}

\newcommand{\pred}[1]{\textsf{\small{#1}}}

\graphicspath{ {images/} }

\forestset{
    .style={
        for tree={
            base=bottom,
            child anchor=north,
            align=center,
            s sep+=1cm,
    straight edge/.style={
        edge path={\noexpand\path[\forestoption{edge},thick,-{Latex}] 
        (!u.parent anchor) -- (.child anchor);}
    },
    if n children={0}
        {tier=word, draw, thick, rectangle}
        {draw, diamond, thick, aspect=2},
    if n=1{%
        edge path={\noexpand\path[\forestoption{edge},thick,-{Latex}] 
        (!u.parent anchor) -| (.child anchor) node[pos=.2, above] {Y};}
        }{
        edge path={\noexpand\path[\forestoption{edge},thick,-{Latex}] 
        (!u.parent anchor) -| (.child anchor) node[pos=.2, above] {N};}
        }
        }
    }
}

\title{Analyzing Differentiable Fuzzy Implications}

\author{Emile van Krieken\and Erman Acar\and Frank van Harmelen \\
        \affiliations Vrije Universiteit Amsterdam, \\
        De Boelelaan 1105, 1081 HV Amsterdam, The Netherlands \\
        \emails \{e.van.krieken, Frank.van.Harmelen, Erman.Acar\}@vu.nl}

\begin{document}
\maketitle

\begin{abstract}
Combining symbolic and neural approaches has gained considerable attention in the AI community, as it is often argued that the strengths and weaknesses of these approaches are complementary. One such trend in the literature are weakly supervised learning techniques that employ operators from fuzzy logics. In particular, they use prior background knowledge described in such logics to help the training of a neural network from unlabeled and noisy data. By interpreting logical symbols using neural networks (or \emph{grounding} them), this background knowledge can be added to regular loss functions, hence making reasoning a part of learning. 

In this paper, we investigate how implications from the fuzzy logic literature behave in a differentiable setting. In such a setting, we analyze the differences between the formal properties of these fuzzy implications. It turns out that various fuzzy implications, including some of the most well-known, are highly unsuitable for use in a differentiable learning setting.
A further finding shows a strong imbalance between gradients driven by the antecedent and the consequent of the implication. Furthermore, we introduce a new family of fuzzy implications (called sigmoidal implications) to tackle this phenomenon. Finally, we empirically show that it is possible to use \dfuzz for semi-supervised learning, and show that sigmoidal implications outperform other choices of fuzzy implications. 
\end{abstract}

\section{Introduction}
In recent years, integrating  symbolic and statistical 
approaches to AI gained considerable attention \pcite{garcez2012neural,Besold2017a}. This research line has gained further traction due to recent influential critiques on purely statistical deep learning \pcite{marcus2018deep,pearl2018theoretical}. While deep learning has brought many important breakthroughs \pcite{brock2018large,radford2019language,silver2017mastering}, there have been concerns about the massive amounts of data that models need to learn even a  simple concept. In contrast, traditional symbolic AI could easily reuse concepts e.g., using only a single logical statement one can already express domain knowledge conveniently.

However, symbolic AI also has its weaknesses. One is scalability: dealing with large amounts of data while performing complex reasoning tasks. Another is not being able to deal with the noise and ambiguity of e.g. sensory data. The latter is also related to the well-known \textit{symbol grounding problem} which \tcite{harnad1990symbol} defines as how \say{the semantic interpretation of a formal symbol system can be made intrinsic to the system, rather than just parasitic on the meanings in our heads}. In particular, symbols refer to concepts that have an intrinsic meaning to us humans, but computers manipulating these symbols cannot  \emph{understand} (or \emph{ground}) this meaning. On the other hand, a properly trained deep learning model excels at modeling complex sensory data. Therefore, several recent approaches \pcite{diligenti2017,garnelo2016towards,Serafini2016,DBLP:conf/nips/2018,Evans2018} aimed at interpreting symbols that are used in logic-based systems using deep learning models. These implement \tcite{harnad1990symbol} \say{a hybrid nonsymbolic/symbolic system [...] in which the elementary symbols are grounded in [...] non-symbolic representations that pick out, from their proximal sensory projections, the distal object categories to which the elementary symbols refer.} 

In this article, we introduce \textit{\dfuzz} (\dfl) which aims to integrate reasoning and learning by using logical formulas expressing background knowledge. In order to ensure loss functions are differentiable, \dfl uses fuzzy logic semantics \pcite{klir1995fuzzy}. Moreover, predicate, function and constant symbols are interpreted using a deep learning model. By maximizing the degree of truth of the background knowledge using gradient descent, both learning and reasoning are performed in parallel. The loss function can be used for weakly supervised learning \pcite{zhou2017}, like detecting noisy or inaccurate supervision \pcite{Donadello2017}, or semi-supervised learning \pcite{pmlr-v80-xu18h,P16-1228}. Under this interpretation, \dfl corrects the predictions of the deep learning model when it is logically inconsistent. 

Next, we present an analysis of the choice of fuzzy implication. A fuzzy implication generalizes the Boolean implication, and it is usually differentiable, which enables its use in \dfl. Interestingly, the derivatives of the implications determine how \dfl corrects the deep learning model when its predictions are inconsistent with the background knowledge. We show that the qualitative properties of these derivatives are integral to both the theory and practice of \dfl. 


More specifically, the main contribution of this article is to answer the following question: \emph{Which fuzzy logic implications have convenient theoretical properties when using them in gradient descent?} To this end, 
\begin{itemize}
    \item we introduce several known implications from fuzzy logic (Section \ref{sec:background}) and the framework of \dfuzz (Section \ref{sec:reallogic}) that uses these implications;
    \item we analyze the theoretical properties of fuzzy implications and introduce a new family of fuzzy implications called sigmoidal implications (Section \ref{sec:implicationss});
    \item we perform experiments to compare fuzzy implications in a semi-supervised experiment (Section \ref{chapter:experiments}). 
    \item we conclude with several recommendations for choices of fuzzy implications.
\end{itemize}

\section{Background}
\label{sec:background}

 We will denote predicates using $\pred{cushion}$, variables by $x, y, z, x_1, ...$ and objects by $o_1, o_2, ...,$. For convenience, we will be limiting ourselves to function-free formulas in prenex normal form. Formulas in \textit{prenex normal form} start with quantifiers followed by a quantifier-free subformula. 
 An \textit{atom} is $\predP(t_1, ..., t_m)$ where $t_1, ..., t_m$ are terms. If $t_1, ..., t_m$ are all constants, we say it is a \textit{ground atom}.

Fuzzy logic is a real-valued logic where truth values are real numbers in $[0, 1]$ where 0 denotes completely false and 1 denotes completely true. 
We will be looking at predicate fuzzy logics in particular, which extend propositional fuzzy logics with universal and existential quantification. In this text, we limit ourselves to the classic fuzzy negation $N(a)=1-a$.

To properly introduce fuzzy implications, we require the notions of t-norms that generalize boolean conjunction, and t-conorms that generalize boolean disjunction. 
A \textit{t-norm} is a function $T: [0,1]^2\rightarrow [0, 1]$ that is commutative, associative, increasing, and for all $a\in [0,1]$, $T(1, a) = a$. A \textit{t-conorm} is a function $S: [0,1]^2\rightarrow [0, 1]$ that is commutative, associative, increasing, and for all $a\in [0,1]$, $S(0, a) = a$. T-conorms are constructed from a t-norm using $S(a, b) = 1 - T(1-a, 1-b)$.

Fuzzy implications are used to compute the truth value of $p\rightarrow q$. $p$ is called the \textit{antecedent} and $q$ the \textit{consequent} of the implication. We follow \tcite{Jayaram2008} and refer to it for details and proofs.
\begin{deff}
\label{def:implication}
A \textit{fuzzy implication} is a function $I: [0, 1]^2\rightarrow [0, 1]$ so that for all $a, c\in [0, 1]$, $I(\cdot, c)$ is decreasing, $I(a, \cdot)$ is increasing and for which $I(0, 0) = 1$,  $I(1, 1) = 1$ and $I(1, 0) = 0$.
\end{deff}
From this definition follows that $I(0, 1) = 1$. We next introduce several optional properties of fuzzy implications that we will use in our analysis. 
\begin{deff}
\label{deff:implications_optional}
A fuzzy implication $I$ satisfies
\begin{enumerate}
    \item \textit{left-neutrality (LN)} if for all $c\in [0,1]$, $I(1, c) = c$ (generalizes $(1\rightarrow p) \equiv p$);
    \item the \textit{exchange principle (EP)} if for all $ a,b,c\in[0,1]$,  $I(a, I(b, c)) = I(b, I(a, c))$ (generalizes $p\rightarrow(q\rightarrow r) \equiv q\rightarrow(p\rightarrow r)$);
    \item the \textit{identity principle (IP)} if for all $a\in[0,1]$, $I(a, a) = 1$ (generalizes the tautology $p\rightarrow p$);
    \item \textit{contrapositive symmetry (CP)} if for all $a, c\in [0,1]$, $I(a, c)=I(1-c, 1-a)$  (generalizes $p\rightarrow q \equiv \neg q \rightarrow \neg p$);
    \item \textit{left-contrapositive symmetry (L-CP)} if for all $a, c\in [0,1]$, $I(1-a, c) = I(1-c, a)$ (generalizes $\neg p \rightarrow q \equiv \neg q \rightarrow p$);
    \item \textit{right-contrapositive symmetry (R-CP)} if for all $a,c\in[0,1]$, $I(a, 1-c) = I(c, 1 - a)$ (generalizes $p\rightarrow \neg q \equiv q \rightarrow \neg p$).
\end{enumerate}
\end{deff}
\subsubsection{R-Implications}
Using a common construction, we find R-implications. They are the standard choice for implication in t-norm fuzzy logics. 



\begin{deff}
\label{deff:r-implication}
Let $T$ be a t-norm. The function $I_T: [0,1]^2\rightarrow [0, 1]$ is called an \textit{R-implication} and defined as $I_T(a, c) = \sup\{b\in [0, 1]|T(a, b) \leq c\}$.
\end{deff}
The \textit{supremum} of a set $A$, denoted $\sup\{A\}$, is the lowest upper bound of $A$. All R-implications are fuzzy implications, and all satisfy LN, IP and EP. 
Note that if $a\leq c$ then $I_T(a, c) = 1$. 

\subsubsection{S-Implications}
In classical logic, the (material) implication is defined using $p\rightarrow q = \neg p \vee q$. 
Generalizing this definition, we can use a t-conorm $S$ to construct a fuzzy implication.
\begin{deff}
Let $S$ be a t-conorm. The function $I_{S}: [0, 1]^2\rightarrow[0,1]$ is called an \textit{S-implication} and is defined for all $a, c\in [0, 1]$ as $I_S(a, c) = S(1 - a, c)$.
\end{deff}

\begin{table*}[ht]
    \centering
    \begin{tabular}{lllll}
    \hline
    Name                  & Associated t-norm &  Equation                          & Properties\\ \hline 
    Kleene-Dienes         & Gödel             & $I_{KD}(a, c) = \max(1-a, c)$     & All but IP, S-implication \\
    Reichenbach           & Product           & $I_{RC}(a, c) = 1 - a + a\cdot c$ & All but IP, S-implication \\ 
    \luk                  & \luk              & $I_{LK}(a, c) = \min(1-a+c, 1)$   & All, S-implication, R-implication\\ 
    Gödel                 & Gödel             & $I_G(a, c) =\begin{cases}
        1, & \text{if } a \leq c \\
        c, & \text{otherwise}
      \end{cases}$ & LN, EP, IP, R-implication \\
    Goguen               & Product            & $I_{GG}(a, c) =\begin{cases}
        1, & \text{if } a \leq c \\
        \frac{c}{a}, & \text{otherwise}
      \end{cases}$ & LN, EP, IP, R-implication \\
    \hline
    
    \end{tabular}
    \caption{Some common differentiable implications.}
    \label{tab:implications}
\end{table*}
All S-implications $I_{S}$ are fuzzy implications and satisfy every property from Definition \ref{deff:implications_optional} but IP.

Table \ref{tab:implications} shows some common differentiable S-implications and R-implications.

\label{chapter:theory}

\section{Differentiable Fuzzy Logics}
\label{sec:reallogic}

\dfuzz (\dfl) are fuzzy logics with differentiable connectives for which differentiable loss functions can be constructed that represent logical formulas. Examples of logics in this family \pcite{Serafini2016,marra2019learning,diligenti2017,marra2018,guo2016} will be discussed in Section \ref{chapter:related_work}. They use background knowledge to deduce the truth value of statements in unlabeled or poorly labeled data to be able to use such data during learning. This can be beneficial as unlabeled, poorly labeled and partially labeled data is cheaper and easier to come by. 

We motivate the use of \dfl with the following scenario: Assume we have an agent $M$ whose goal is to describe the scene on an image. It gets feedback from a supervisor $S$, 
who does not have an exact description of these images available. However, $S$ does have a background knowledge base $\corpus$, encoded in some logical formalism, about the concepts contained on the images. The intuition behind \dfl is that $S$ can correct $M$'s descriptions of scenes when they are not consistent with its knowledge base $\corpus$.

\begin{figure}
    \centering
    \includegraphics[width=0.33\textwidth]{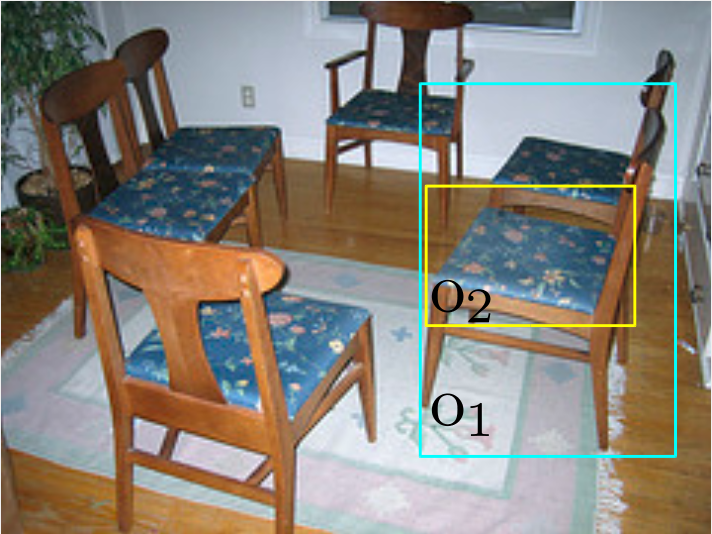}
    \caption{In this running example, we have an image with two objects on it, $o_1$ and $o_2$. }
    \label{fig:chair2}
\end{figure}

\begin{exmp}
\label{exmp:diffreason}
`Agent $M$ has to describe the image $I$ in Figure \ref{fig:chair2} containing two objects, $o_1$ and $o_2$. $M$ and the supervisor $S$ only know about the unary class predicates $\{\pred{chair}, \pred{cushion}, \pred{armRest}\}$ and the binary predicate $\{\pred{partOf}\}$. Since $S$ does not have a description of $I$, it will have to correct $M$ based on the knowledge base $\corpus$. $M$ describes the image as follows, where the probability indicates the confidence in an observation: 
\begin{align*}
p(\pred{chair}(o_1))&=0.9  &p(\pred{chair}(o_2))&=0.4\\
p(\pred{cushion}(o_1))&=0.05 & p(\pred{cushion}(o_2))&=0.5\\
p(\pred{armRest}(o_1))&=0.05 & p(\pred{armRest}(o_2))&=0.1\\
p(\pred{partOf}(o_1, o_1)) &= 0.001 & p(\pred{partOf}(o_2, o_2)) &= 0.001\\
p(\pred{partOf}(o_1, o_2)) &= 0.01 & p(\pred{partOf}(o_2, o_1)) &= 0.95
\end{align*}
Suppose that $\corpus$ contains the following logic formula which says objects that are a part of a chair are either cushions or armrests:
\begin{equation*}
    \forall x, y\ \pred{chair}(x) \wedge \pred{partOf}(y, x) \rightarrow \pred{cushion}(y) \vee \pred{armRest}(y).
\end{equation*}
$S$ might now reason that since $M$ is relatively confident of $\pred{chair}(o_1)$ and $\pred{partOf}(o_2, o_1)$ that the antecedent of this formula is satisfied, and thus $\pred{cushion}(o_2)$ or $\pred{armRest}(o_2)$ has to hold. Since $p(\pred{cushion}(o_2)|I, o_2) > p(\pred{armRest}(o_2)|I, o_2)$, a possible correction would be to tell $M$ to increase its degree of belief in $\pred{cushion}(o_2)$.
\end{exmp}

We would like to automate the kind of supervision $S$ performs in the previous example. Therefore, we next formally introduce \dfl, in which truth values of ground atoms are in $[0, 1]$, and logical connectives are interpreted using fuzzy operators. 
\dfl defines a new semantics using vector embeddings and functions on such vectors in place of classical semantics. In classical logic, a \textit{structure} consists of a domain of discourse and an interpretation function, and is used to give meaning to the predicates. Similarly, in \dfl a structure consists of a probability distribution defined on an embedding space and an \textit{embedded interpretation}:

\begin{deff}
\label{deff:distr_inter}

    A \textit{\dfuzz structure} is a tuple $\langle p, \interpretation\rangle$, where $p$ is a \textit{domain distribution} over $d$-dimensional 
    objects $o\in \mathbb{R}^d$ whose \textit{domain of discourse} is the \textit{support} of $p$ (i.e., $\objects=\mathrm{supp}(p)=\{o|p(o) > 0, o\in \mathbb{R}^d\}$), and  $\interpretation$ is an (\textit{embedded}) \textit{interpretation}  under $\btheta$ (a $W$-dimensional real vector, also called \emph{parameters})  which maps every  predicate symbol $\pred{P} \in \predicates$ with arity $\alpha$ to a function that associates $\alpha$ objects to an element in $[0, 1]$ (i.e., $\interpretation(\pred{P}): \objects^\alpha\rightarrow [0, 1]$).
\end{deff}

To address the \textit{symbol grounding problem} \pcite{harnad1990symbol}, objects in the domain of discourse are $d$-dimensional vectors of reals. Their semantics come from the underlying semantics of the vector space as terms are interpreted in a real (valued) world \pcite{Serafini2016}. Predicates are interpreted as functions mapping these vectors to a fuzzy truth value. Embedded interpretations are implemented using a neural network model with trainable network parameters $\btheta$. Different values of $\btheta$ will produce different embedded interpretations $\interpretation$. The domain distribution is used to limit the size of the vector space. For example, $p$ might be the distribution over images representing only the natural images.  

Next, we define how to compute the truth value of formulas of \dfl, which generalizes the computation of Real Logic \pcite{Serafini2016}. 
An \textit{aggregation operator} is a function $A: [0, 1]^n\rightarrow [0, 1]$ that is symmetric and increasing with respect to each argument, and for which $A(0, ..., 0)=0$ $A(1, ..., 1) = 1$.
A \textit{variable assignment} $\mu$ maps variable symbols $x$ to objects $o\in \objects$. $\mu(x)$ retrieves the object $o\in \objects$ assigned to $x$ in $\mu$.

\begin{deff}
\label{deff:val}
Let $\langle p, \interpretation\rangle$ be a \dfl structure, $T$ a t-norm, $S$ a t-conorm, $I$ a fuzzy implication and $A$ an aggregation operator. Then the \textit{valuation function} $\fullval$ (or, for brevity, $\val$) computes the truth value of a formula $\varphi$ in $\fol$ given a variable assignment $\instantiation$. It is defined inductively as follows:
\begin{align}
    \label{eq:rlpred}
    &\val\left(\pred{P}(x_1, ..., x_m) \right) = \interpretation(\pred{P})\left(\instantiation(x_1 ), ..., \instantiation(x_m )\right)\\
    \label{eq:rlneg}
    &\val(\neg \phi ) = 1 - \val(\phi )\\
    \label{eq:rlconj}
    &\val(\phi \wedge \psi ) = T(\val(\phi ), \val(\psi ))\\
    &\val(\phi\vee\psi ) = S(\val(\phi ), \val(\psi ))\\
    \label{eq:rlimp}
    &\val(\phi\rightarrow\psi ) = I(\val(\phi ), \val(\psi ))\\
    \label{eq:rlaggr}
    &\val(\forall x\ \phi ) = A_{o\in \objects} \val(\phi),  \text{with } x \text{ assigned to } o \text{ in } \mu.
\end{align}
\end{deff}

Equation \ref{eq:rlpred} defines the fuzzy truth value of an atomic formula. $\instantiation$ finds the objects assigned to the terms $x_1, ..., x_m$ resulting in a list of $d$-dimensional vectors. These are the inputs to the interpretation of the predicate symbol $\interpretation(\pred{P})$ to get a fuzzy truth value. Equations \ref{eq:rlneg} - \ref{eq:rlimp} define the truth values of the connectives using the operators $T, S$ and $I$. 

Equation \ref{eq:rlaggr} defines the truth value of universally quantified formulas $\forall x\ \phi$. This is done by enumerating the domain of discourse $o\in\objects$, computing the truth value of $\phi$ with $o$ assigned to $x$ in $\mu$, and combining the truth values using an aggregation operator $A$.
When enumerating the objects is not viable, we can choose to sample a batch of objects to approximate the computation of the valuation.
It is commonly assumed in Machine Learning \cite{goodfellow2016deep}(p.109) that a dataset contains independent samples from the domain distribution $p$ and thus using such samples approximates sampling from $p$. Unfortunately, by relaxing quantifiers in this way we lose soundness of the logic. 

In \dfl, the parameters $\btheta$ are learned using \textit{fuzzy maximum satisfiability} \pcite{Donadello2017}, which finds parameters that maximize the valuation of the knowledge base $\corpus$.

\begin{deff}
Let $\corpus$ be a knowledge base of formulas, $\langle p, \interpretation\rangle$ a \dfl structure for the predicate symbols in $\corpus$ and $\fullval$ a valuation function. Then the \textit{\dfuzz loss} $\loss_{\dfl}$ of a knowledge base of formulas $\corpus$ is computed using 
\begin{equation}
\label{eq:lossrl}
    \loss_{\dfl}(\btheta; \objects, \corpus) = -\sum_{\varphi\in\corpus}w_{\varphi}\cdot \fullval(\varphi),
\end{equation}
where $w_{\varphi}$ is the weight for formula $\varphi$ which denotes the importance of the formula $\varphi$ in the loss function. The \textit{fuzzy maximum satisfiability problem} is the problem of finding parameters $\btheta^*$ that minimize Equation \ref{eq:lossrl}:
\begin{equation}
\label{eq:bestsatproblem}
    \btheta^* = \text{argmin}_{\btheta}\ \loss_{\dfl}(\btheta; \objects, \corpus).
\end{equation}
\end{deff}

This optimization problem can be solved using a gradient descent method. If the operators $T, S, I$ and $A$ are all differentiable, we can repeatedly apply the chain rule, i.e. reverse-mode differentiation, on the \dfl loss $\loss_{\dfl}(\btheta_n; \objects, \corpus)$, $n=0, ..., N$. This procedure finds the derivative with respect to the truth values of the ground atoms $\frac{\partial \loss_{\dfl}(\btheta_n; \objects, \corpus)}{\partial \interpretationfol_{\btheta_n}(\predP)(o_1, ..., o_m)}$. We can use these partial derivatives to update the parameters $\btheta_n$ using the chain rule, resulting in a different embedded interpretation $\interpretationfol_{\btheta_{n+1}}$. 

One particularly interesting property of \dfuzz is that the partial derivatives of the subformulas with respect to the satisfaction of the knowledge base have a somewhat explainable meaning. For example, turning back to Example \ref{exmp:diffreason}, the computed partial derivatives reflect whether we should increase $p(\pred{cushion}(o_2))$, that is, increase the agents belief in $\pred{cushion}(o_2)$. 

\section{Differentiable Fuzzy Implications}
\label{sec:implicationss}
A significant proportion of background knowledge is written as universally quantified implications of the form  $\forall x\ \phi(x) \rightarrow \psi(x)$, like `all humans are mortal'. 
The implication is used in two well known rules of inference. \textit{Modus ponens} inference says that if $\forall x\ \phi(x) \rightarrow \psi(x)$ and we know that $\phi(x)$, then also $\psi(x)$. \textit{Modus tollens} inference says that if $\forall x\ \phi(x) \rightarrow \psi(x)$ and we know that $\neg\psi(x)$, then also $\neg\phi(x)$, as if $\phi(x)$ were true, $\psi(x)$ should also have been. 

When the learning agent predicts a scene in which an implication is false, the supervisor has multiple choices to correct it. Consider the implication `all ravens are black'. There are 4 categories for this formula: \textit{black ravens} (BR), \textit{non-black non-ravens} (NBNR), \textit{black non-ravens} (BNR) and \textit{non-black ravens} (NBR). Assume our agent observes an NBR, which is inconsistent with the background knowledge. There are four options to consider.
\begin{enumerate}
    \item \textit{Modus Ponens} (MP): The antecedent is true, so by modus ponens, the consequent is also true. We trust the agent's observation of a raven and believe it was a black raven (BR).
    \item \textit{Modus Tollens} (MT): The consequent is false, so by modus tollens, the antecedent is also false. We trust the agent's observation of a non-black object and believe it was not a raven (NBNR).
    \item \textit{Distrust}: We believe the agent is wrong both about observing a raven and a non-black object and it was a black object which is non-raven (BNR).
    \item \textit{Exception}: We trust the agent and ignore the fact that its observation goes against the background knowledge. Hence, it has to be a non-black raven (NBR).
\end{enumerate}
The distrust option seems somewhat useless. The exception option can be correct, but we cannot know when there is an exception from the agent's observations alone. In such cases, \dfl would not be very useful since it would not teach the agent anything new.

We can safely assume that there are far more non-black objects which are not ravens than there are ravens. Thus, from a statistical perspective, it is most likely that the agent observed an NBNR. This shows the imbalance associated with the implication, which was first noted in \tcite{vankrieken2019ravens} for the Reichenbach implication. It is quite similar to the \textit{class imbalance problem} in Machine Learning \pcite{japkowicz2002class} in that the real world has far more `negative' (or \textit{contrapositive}) examples than positive examples of the background knowledge.

This problem is closely related to the Raven paradox \pcite{Hempel1945,Vranas2004} from the field of confirmation theory which ponders what evidence can confirm a statement like `ravens are black'. It is usually stated as follows:
\begin{itemize}
    \item Premise 1: Observing examples of a statement contributes positive evidence towards that statement.
    \item Premise 2: Evidence for some statement is also evidence for all logically equivalent statements.
    \item Conclusion: Observing examples of non-black non-ravens is evidence for `all ravens are black'.
\end{itemize}
The conclusion follows from the fact that `non-black objects are non-ravens' is logically equivalent to `ravens are black'. Although we are considering logical validity instead of confirmation, we note that for \dfl a similar thing happens. When we correct the observation of an NBR to a BR, the difference in truth value is equal to when we correct it to NBNR. More precisely, representing `ravens are black' as $I(a, b)$, where, for example, $I(1, 1)$ corresponds to BR:
\begin{align*}
  &A(x_1, ..., I(1, 0), ..., x_n) - A(x_1, ..., I(1, 1), ..., x_n) \\
  = &A(x_1, ..., I(1, 0), ..., x_n) - A(x_1, ..., I(0, 0), ..., x_n)  
\end{align*}
as $I(0, 0) = I(1, 1) = 1$.  Furthermore, when one agent observes a thousand BR's and a single NBR, and another agent observes a thousand NBNR's and a single NBR, their truth value for `ravens are black' is equal. This seems strange, as the first agent has actually seen many ravens of which only a single exception was not black, while the second only observed many non ravens which were not black, among which a single raven that was not black either. Intuitively, the first agent's beliefs seem to be more in line with the background knowledge. We will now proceed to analyse a number of implication operators in light of this discussion. 



\subsection{Analyzing the Implication Operators}
We define two functions for a fuzzy implication $I$:
\begin{align}
    \dmp(a, c) &= \frac{\partial I(a, c)}{\partial c} \\
    \label{eq:dmp}
    \dmt(a, c) &= -\frac{\partial I(a, c)}{\partial a} = \frac{\partial I(a, c)}{\partial \neg a}.
\end{align}
$\dmp$ is the derivative with respect to the consequent and $\dmt$ is the derivative with respect to the \textit{negated} antecedent. We choose to take the derivative with respect to the negated antecedent as it makes it easier to compare them.

\begin{deff}
A fuzzy implication $I$ is called
\textit{contrapositive differentiable symmetric} if $\dmp(a, c) = \dmt(1-c, 1-a)$ for all $a, c \in [0, 1]$.

\end{deff}
A consequence of contrapositive differentiable symmetry is that if $c=1-a$, then the derivatives are equal since $\dmp(a, c)=\dmt(1 - c, 1-a) = \dmt(1 -(1 - a), c) = \dmt(a, c)$. This could be seen as the `distrust' option in which it increases the consequent and negated antecedent equally.

\begin{prop}
If a fuzzy implication $I$ is contrapositive symmetric, it is also contrapositive differentiable symmetric. 
\end{prop}
\begin{proof}
Say we have an implication $I$ that is contrapositive symmetric. We find that $\dmp(a, c) = \frac{\partial I(a, c)}{\partial c}$ and $\dmt(1-c, 1-a) = -\frac{\partial I(1-c, 1-a)}{\partial 1 - c}$. Because $I$ is contrapositive symmetric, $I(1-c, 1-a) = I(a, c)$. Thus, $\dmt(1-c, 1-a) = -\frac{\partial I(a, c)}{\partial 1-c}=\frac{\partial I(a, c)}{\partial c}=\dmp(a, c)$. 
\end{proof}
In particular, by this proposition all S-implications are contrapositive differentiable symmetric. This says that there is no difference in how the implication handles the derivatives with respect to the consequent and antecedent.

\begin{prop}
\label{prop:diff_left_neutral}
If an implication $I$ is left-neutral, then $\dmp(1, c) = 1$. If, in addition, $I$ is contrapositive differentiable symmetric, then $\dmt(a, 0) = 1$.
\end{prop}
\begin{proof}
First, assume $I$ is left-neutral. Then for all $c\in[0,1]$, $I(1, c) = c$. Taking the derivative with respect to $c$, it turns out that $\dmp(1, c) = 1$. Next, assume $I$ is contrapositive differentiable symmetric. Then, $\dmp(1, c) = \dmt(1 - c, 1-1) = \dmt(1-c, 0) = 1$. As $1-c\in[0, 1]$, $\dmt(a, 0)=1$.
\end{proof}

All S-implications and R-implications are left-neutral, but only S-implications are all also contrapositive differentiable symmetric.
The derivatives of R-implications vanish when $a\leq c$, that is, on no less than half of the domain. 
Note that the plots in this section are rotated so that the smallest value is in the front. In particular, plots of the derivatives of the implications are rotated 180 degrees compared to the implications themselves. 

\subsubsection{Gödel-based Implications}
\dualfigure{\imgT I_godel.pdf}{\imgT I_kleene.pdf}{Left: The Gödel implication. Right: The Kleene Dienes implication.}{fig:kd_godel}

Implications based on the Gödel t-norm ($T_G(a, b)=\min(a, b)$) make strong discrete choices and increase at most one of their outputs. The two associated implications are shown in Figure \ref{fig:kd_godel}. The Gödel implication is a simple R-implication with the following derivatives:
\begin{equation}
\label{eq:god_imp_deriv}
    \dmpa{I_G}(a, c) =  \begin{cases}
        1, & \text{if } a > c \\
        0, & \text{otherwise}
      \end{cases}, \quad
      \dmta{I_G} (a, b)= 0.
\end{equation}

The Gödel implication increases the consequent whenever $a > c$, and the antecedent is never changed. This makes it a poorly performing implication in practice. For example, consider $a=0.1$ and $c=0$. Then the Gödel implication increases the consequent, even if the agent is fairly certain that neither is true. Furthermore, as the derivative with respect to the negated antecedent is always 0, it can never choose the modus tollens correction, which, as we argued, is actually often the best choice. 

The derivatives of the Kleene-Dienes implication are
\begin{align}
    \dmpa{I_{KD}}(a, c) &=  \begin{cases}
        1, & \text{if } 1 - a < c \\
        0, & \text{if } 1 - a > c
      \end{cases},\\
      \dmta{I_{KD}} (a, b)&=  \begin{cases}
        1, & \text{if } 1-a > c \\
        0, & \text{if } 1-a < c
      \end{cases}   .
\end{align}
Or, simply put, if we are more confident in the truth of the consequent than in the truth of the negated antecedent, increase the truth of the consequent. Otherwise, decrease the truth of the antecedent. This decision can be somewhat arbitrary and does not take into account the imbalance of modus ponens and modus tollens.




\subsubsection{\luk\ Implication}
The \luk\ implication is both an S- and an R-implication. It has the simple derivatives 
\begin{equation}
\label{eq:impl_deriv_luk}
    \dmpa{I_{LK}}(a, c) = \dmta{I_{LK}}(a, c) =  \begin{cases}
        1, & \text{if } a > c \\
        0, & \text{if } a < c.
      \end{cases}
\end{equation}
Whenever the implication is not satisfied because the antecedent is higher than the consequent, it simply increases the negated antecedent and the consequent until it is lower. This could be seen as the `distrust' choice as both observations of the agent are equally corrected, and so does not take into account the imbalance between modus ponens and modus tollens cases. The derivatives of the Gödel implication $I_G$ are equal to those of $I_{LK}$ except that $I_G$ always has a zero derivative for the negated antecedent.

\subsubsection{Product-based Implications}
\label{sec:prod-implication}
\dualfigure{\imgT I_goguen.pdf}{\imgT I_reichenbach.pdf}{Left: The Goguen implication. Right: The Reichenbach implication.}{fig:goguen-i}


The product t-norm is given as $T_P(a, b)= a\cdot b$. The associated R-implication is called the Goguen implication. We plot this implication in Figure \ref{fig:goguen-i}. 
The derivatives of $I_{GG}$ are
\begin{align}
    \dmpa{I_{GG}}(a, c) &= \begin{cases}
    0, &\text{if } a\leq c \\
    \frac{1}{a}, &\text{otherwise} 
    \end{cases}, \\
    \dmta{I_{GG}}(a, c) &= \begin{cases}
    0, &\text{if } a\leq c \\
    \frac{c}{a^2}, &\text{otherwise}
    \end{cases}.
\end{align}

\dualfigure{\imgT dMP_logscale_goguen.pdf}{\imgT dMT_logscale_goguen.pdf}{The derivatives of the Goguen implication. Note that we plot these in log scale.}{fig:deriv_goguen}

We plot these in Figure \ref{fig:deriv_goguen}. This derivative is not very useful. First of all, both the modus ponens and modus tollens derivatives increase with $\neg a$. This is opposite of the modus ponens rule as when the antecedent is \textit{low}, it increases the consequent most. For example, if $\pred{raven}$ is 0.1 and $\pred{black}$ is 0, then the derivative with respect to $\pred{black}$ is 10, because of the singularity when $a$ approaches 0.

The derivatives of the Reichenbach implication are given by:
\begin{equation}
    \dmpa{I_{RC}}(a, c) = a, \quad \dmta{I_{RC}}(a, c) = 1-c.
\end{equation}

These derivatives closely follow modus ponens and modus tollens inference. When the antecedent is high, increase the consequent, and when the consequent is low, decrease the antecedent. However, around $(1-a)=c$, the derivative is equal and the `distrust' option is chosen. This can result in counter-intuitive behaviour. For example, if the agent predicts 0.6 for $\pred{raven}$ and 0.5 for $\pred{black}$ and we use gradient descent until we find a maximum, we could end up at 0.3 for $\pred{raven}$ and 1 for $\pred{black}$. We would end up increasing our confidence in $\pred{black}$ as $\pred{raven}$ was high. However, because of additional modus tollens reasoning, $\pred{raven}$ is barely true. 

Furthermore, if the agent mostly predicts values around $a=0,\ c=0$ as a result of the modus tollens case being the most common, then a majority of the gradient decreases the antecedent as $\dmta{I_{RC}}(0,0)= 1$. We next identify two methods that counteract this behavior.

\dualfigure{\imgT dMT_reichenbach.pdf}{\imgT log_dMT_reichenbach.pdf}{Left: The antecedent derivative of the Reichenbach implication. Right: The antecedent derivative of the Reichenbach implication  with the log-product aggregator.}{fig:reichenbach-aggregators}

\subsubsection{Log product aggregator} The first method for counteracting the `corner' behavior notes that different aggregators change how the derivatives of the implications behave. Note that the aggregator based on the product t-norm is $A_P(x_1, ..., x_n)=\prod_{i=1}^nx_i$. As formulas are in prenex normal form, maximizing this aggregator is equivalent to maximizing the logarithm of this aggregator, which gives $A_{\log P}(x_1, ..., x_n)=\sum_{i=1}^n \log(x_i)$ that is reminiscent of the cross-entropy loss function. Using the chain rule, we find that the negated antecedent derivative becomes:
\begin{align}
    \frac{\partial A_{\log P}(I(a_1, c_1), ..., I(a_n, c_n))}{\partial 1- a_i}= \frac{\dmta{I}(a_i, c_i)}{I(a_i, c_i   )}
\end{align}
As this divides by the truth value of the implication, implications that do not have a high truth value get stronger derivatives. We plot the negated antecedent derivative for the Reichenbach implication when using the log-product aggregator in Figure \ref{fig:reichenbach-aggregators}. Note that the derivative with respect to the negated antecedent in $a_i=0$, $c_i=0$ is still 1. By differentiable contrapositive symmetry, the consequent derivative is 0. Therefore, when using the log-product aggregator, one of antecedent and consequent will still have a gradient.

\subsubsection{Sigmoidal Implications}
\label{sec:sigm_implication}
For the second method for tackling the corner problem, we introduce a new class of fuzzy implications formed by transforming other fuzzy implications using the sigmoid function and translating it so that the boundary conditions still hold.\footnote{The derivation, along with several proofs of properties, can be found at \url{https://github.com/HEmile/differentiable-fuzzy-logics/blob/master/appendix_sigmoidal_implications.pdf}.}

\begin{deff}
\label{deff:sigmoidal}
If $I$ is a fuzzy implication, then the $I$-sigmoidal implication $\sigma_I$ is given for some $s>0$ as
\begin{align}
  \sigma_I(a, c)=\frac{\left(1 + e^{\frac{s}{2}}\right) \cdot \sigma\left(s\cdot I(a, c) - \frac{s}{2}\right) - 1}{e^{\frac{s}{2}}-1}
\end{align}
where $\sigma(x)=\frac{1}{1+e^x}$ denotes the sigmoid function.
\end{deff}

\dualfigure{\imgT min05I_probsum9.pdf}{\imgT min05I_probsum001.pdf}{The Reichenbach-sigmoidal implication for different values of $s$. Left: $s=9$. Right: $s=0.01$.}{fig:probsum_sigm_s}

Here $s$ controls the `spread' of the curve. $\sigma_I$ is the function $ \sigma\left(s\cdot \left(I(a, c) - \frac{1}{2}\right)\right)$ linearly transformed so that its codomain is the closed interval $[0, 1]$. $\sigma_I$ is a fuzzy implication in the sense of Definition \ref{def:implication}. Furthermore, $\sigma_I$ satisfies the identity principle if $I$ does, and is contrapositive (differentiable) symmetric if $I$ is. We plot the Reichenbach-sigmoidal implication $\sigma_{I_{RC}}$ in Figure \ref{fig:probsum_sigm_s} for two values of $s$. Note that for $s=0.01$, the plotted function is indiscernible from the plot of the Reichenbach implication in Figure \ref{fig:goguen-i} as the interval on which the sigmoid acts is extremely small and the sigmoidal transformation is almost linear. The derivative is computed as
\begin{align}
\label{eq:deriv_sigm}
\begin{split}
    \frac{\partial \sigma_I(a, c)}{\partial I(a, c)}=& \frac{s\cdot\left(1+e^{\frac{s}{2}}\right)}{e^{\frac{s}{2}}-1} \cdot \sigma\left(s\cdot I(a, c) - \frac{s}{2}\right) \cdot \\
    & \left(1 -  \sigma\left(s\cdot I(a, c) - \frac{s}{2}\right)\right). 
\end{split}
\end{align}
The derivative keeps the properties of the original function but smoothes the gradient for higher values of $s$. As the derivative of the sigmoid function (that is, $\sigma(x)\cdot(1 - \sigma(x))$) cannot be zero, this derivative vanishes only when $\frac{\partial I(a, c)}{\partial \neg a}=0$ or $\frac{\partial I(a, c)}{\partial c}=0$.

\dualfigure{\imgT min05dMP_probsum9.pdf}{\imgT min05dMT_probsum9.pdf}{The derivatives of the Reichenbach-sigmoidal implication for $s=9$.}{fig:deriv_rcsigm}

\dualfigure{\imgT log_dMP_reichenbach.pdf}{\imgT min05log_dMP_probsum9.pdf}{The consequent derivatives of the log-Reichenbach and log-Reichenbach-sigmoidal (with $s=9$) implications. The figure is plotted in log scale.}{fig:log_dmp}

We plot the derivatives for the Reichenbach-sigmoidal implication $\sigma_{I_{RC}}$ in Figure \ref{fig:deriv_rcsigm}. As expected, it is clearly differentiable contrapositive symmetric. Compared to the derivatives of the Reichenbach implication 
it has a small gradient in all corners. 
In Figure \ref{fig:log_dmp} we compare the consequent derivative of the normal Reichenbach implication with the Reichenbach-sigmoidal implication when using the $\log$ product aggregator. A significant difference is that the sigmoidal variant is less `flat' than the normal Reichenbach implication. This can be useful, as this means there is a larger gradient for values of $c$ that make the implication less true. In particular, the gradient at the modus ponens case ($a=1,\ c=1$) and the modus tollens case ($a=0,\ c=0$) are far smaller, which could help balancing the effective total gradient by solving the `corner' problem of the Reichenbach implication. These derivatives are smaller for for higher values of $s$. 

\section{Experiments}
\label{chapter:experiments}

\label{sec:mnist}
To get an idea of the practical behavior of these implications we now perform a series of simple experiments to analyze them in practice. In this section, we discuss experiments using the MNIST dataset of handwritten digits \pcite{lecun-mnisthandwrittendigit-2010} to investigate the behavior of different fuzzy operators introduced in this paper. 
\subsection{Measures}
\label{sec:mnist_ldr}
To investigate the performance of the different configurations of \dfl, we first introduce several useful metrics. In this section, we assume we are dealing with formulas of the form \break $\varphi=\forall x_1, ..., x_m\ \phi(x_1, ..., x_m)\rightarrow \psi(x_1, ..., x_m)$.
\begin{deff}
The \textit{consequent magnitude} $\mpmag$ and the \textit{antecedent magnitude} $\mtmag$ for a knowledge base $\corpus$  is defined as the sum of the partial derivatives of the consequent and antecedent with respect to the \dfl loss:
\begin{align}
    \mpmag&= \sum_{\varphi\in\corpus}\sum_{\instantiation\in\instantiations_\varphi}\frac{\partial \val(\varphi)}{\partial \val(\psi)},\\
    \mtmag&= \sum_{\varphi\in\corpus} \sum_{\instantiation\in\instantiations_\varphi}\frac{-\partial \val(\varphi)}{\partial \val(\phi)},
\end{align}
where $\instantiations_\varphi$ is the set of instances of the universally quantified formula $\varphi$ and $\psi$ and $\phi$ are evaluated under instantiation $\instantiation$.
The \textit{consequent ratio} $\mpratio$ is the sum of consequent magnitudes divided by the sum of consequent and antecedent magnitudes: $\mpratio = \frac{\mpmag}{\mpmag + \mtmag}$.
\end{deff}


\begin{deff}
Given a \textit{labeling function} $l$ that returns the truth value of a formula according to the data for instance $\instantiation$, the \textit{consequent and antecedent correctly updated magnitudes} are the sum of partial derivatives for which the consequent or the negated antecedent is true:
\begin{align}
    \mpcorupdate &= \sum_{\varphi\in\corpus}\sum_{\instantiation\in\instantiations_\varphi}  l(\psi, \mu) \cdot \frac{\partial \val(\varphi)}{\partial \val(\psi)},\\
    \mtcorupdate &= \sum_{\varphi\in\corpus}-\sum_{\instantiation\in\instantiations_\varphi}l(\neg \phi, \mu) \cdot \frac{\partial \val(\varphi)}{\partial \val(\phi)}.
\end{align}
The \textit{correctly updated ratio for consequent and antecedent} are $
    \mpupdateratio = \frac{\mpcorupdate}{\mpmag}$ and    $\mtupdateratio = \frac{\mtcorupdate}{\mtmag}.$
\end{deff}
That is, if the consequent is true in the data, we measure the magnitude of the derivative with respect to the consequent. 
The correctly updated ratios quantify what fraction of the updates are going in the right direction. When they approach 1, \dfl will always increase the truth value of the consequent or negated antecedent correctly. When it not close to 1, we are increasing truth values of subformulas that are wrong, thus ideally, we want these measures to be high.

\subsection{Experimental Setup}
We use a knowledge base $\corpus$ of universally quantified logic formulas. There is a predicate for each digit, that is $\pred{zero},\ \pred{one}, ..., \pred{eight}$ and  $\pred{nine}$. For example, $\pred{zero}(x)$ is true whenever $x$ is a handwritten digit labeled with 0. Secondly, there is the binary predicate $\pred{same}$ that is true whenever both its arguments are the same digit. We next describe the formulas we use. 

\begin{enumerate}
\item $ \forall x, y\ \pred{zero}(x)\wedge \pred{zero}(y) \rightarrow \pred{same}(x, y)$, ..., $\forall x, y\ \pred{nine}(x)\wedge \pred{nine}(y) \rightarrow \pred{same}(x, y) $. If both $x$ and $y$ are handwritten zeros, for example, then they represent the same digit. 
\item $ \forall x, y\ \pred{zero}(x) \wedge \pred{same}(x, y) \rightarrow \pred{zero}(y)$, ..., $\forall x, y\ \pred{nine}(x) \wedge \pred{same}(x, y) \rightarrow \pred{nine}(y) $. If $x$ and $y$ represent the same digit and one of them represents zero, then the other one does as well. 
\item $ \forall x, y\ \pred{same}(x, y) \rightarrow \pred{same}(y, x) $. This formula encodes the symmetry of the $\pred{same}$ predicate. 
\end{enumerate}


We split the MNIST dataset so that 1\% of it is labeled and 99\% is unlabeled. We use two models.\footnote{Code is available at \url{https://github.com/HEmile/differentiable-fuzzy-logics}.} Given a handwritten digit $\bx$, the first model $p_\btheta(y|\bx)$ computes the distribution over the 10 possible labels. We use 2 convolutional layers with max pooling, the first with 10 and the second with 20 filters, and two fully connected hidden layers with 320 and 50 nodes and a softmax output layer, which is trained using cross entropy. The probability that $\pred{same}(\bx_1, \bx_2)$ for two handwritten digits $\bx_1$ and $\bx_2$ holds is modeled by $p_\btheta(\pred{same}|\bx_1, \bx_2)$. This takes the 50-dimensional embeddings of $\bx_1$ and $\bx_2$ of the fully connected hidden layer $e_{\bx_1}$ and $e_{\bx_2}$. These are used in a Neural Tensor Network \pcite{Socher2013} with a hidden layer of size 50. It is trained using binary cross entropy on the cross product of the labeled dataset. As there are far more negative examples than positive examples, we undersample the negative examples. The \dfl loss is weighted by the \textit{\dfl weight} $w_{\dfl}$ and added to the other two losses.

For all experiments, we use the aggregator to the product aggregator with \dfl weight of $w_{dfl}=10$, and optimize the logarithm of the truth value. For conjunction, we use the Yager t-norm with $p=2$, defined as $T_Y(a, b) = \max(1 - ((1-a)^p+(1-b)^p)^{\frac{1}{p}}, 0)$.

\subsection{Results}
We analyze the results for different implication operators. The purely supervised baseline has a test accuracy of $95.0\%\pm 0.001$ (3 runs).
We report the accuracy of recognizing digits in the test set. We do learning for at most 100.000 iterations (or until convergence). We also report the consequent ratio $\mpratio$ and the consequent and antecedent correctly updated ratios $\mpupdateratio$ and $\mtupdateratio$. 
We can compute these values during the backpropagation of the \dfl loss on the `unlabeled' dataset. Because it is a split of a labeled dataset, we can access the labels for evaluation.

\subsubsection{Implications}
In Table \ref{table:mnist_S-implication_experiments}, we compare different fuzzy implications. The Reichenbach implication and the \luk\ implication work well, both having an accuracy around 97\%. Using the Kleene Dienes implication surpasses the baseline as well.

\begin{table}
\centering
\begin{tabular}{l....}
\hline
                                & \mc{Accuracy} & \mc{$\mpratio$} & \mc{$\mpupdateratio$} & \mc{$\mtupdateratio$}               \\
\hline                    
$I_{KD}$                        & 96.1          & 0.10            & \bft{0.88}            & 0.97                           \\
$I_{LK}$                        & \bft{97.0}    & 0.5             & 0.03                  & 0.97                           \\
$I_{RC}$                        & 96.9          & 0.08            & 0.85                  & \bft{0.99} \\
\hline
$I_G$                           & 90.6          & 1               & 0.07                     & \emptyt                              \\
$I_{GG}$                        & 94.0          & 0.86            & 0.01                  & 0.97                           \\
\hline
\end{tabular}
\caption{The results using several different S-implications and R-implications. }
\label{table:mnist_S-implication_experiments}
\end{table}

As hypothesized, the Gödel implication and Goguen implication have worse performance than the supervised baseline. While the derivatives of $I_{LK}$ and $I_G$ only differ in that $I_G$ disables the derivatives with respect to negated antecedent, $I_{LK}$ performs among the best but $I_G$ performs among the worst, suggesting that the derivatives with respect to the negated antecedent are required to successfully applying \dfl. Note that all well performing implications are S-implications, which inherently balance derivatives with respect to the consequent and negated antecedent by being contrapositive differentiable symmetric. 

\subsubsection{Reichenbach-Sigmoidal Implication}
\label{sec:mnist_rcsigmoidal}
The newly introduced Reichenbach-sigmoidal implication $\sigma_{I_{RC}}$ is a promising candidate for the choice of implication. 
\begin{figure}
\centering
\includegraphics[width=\linewidth]{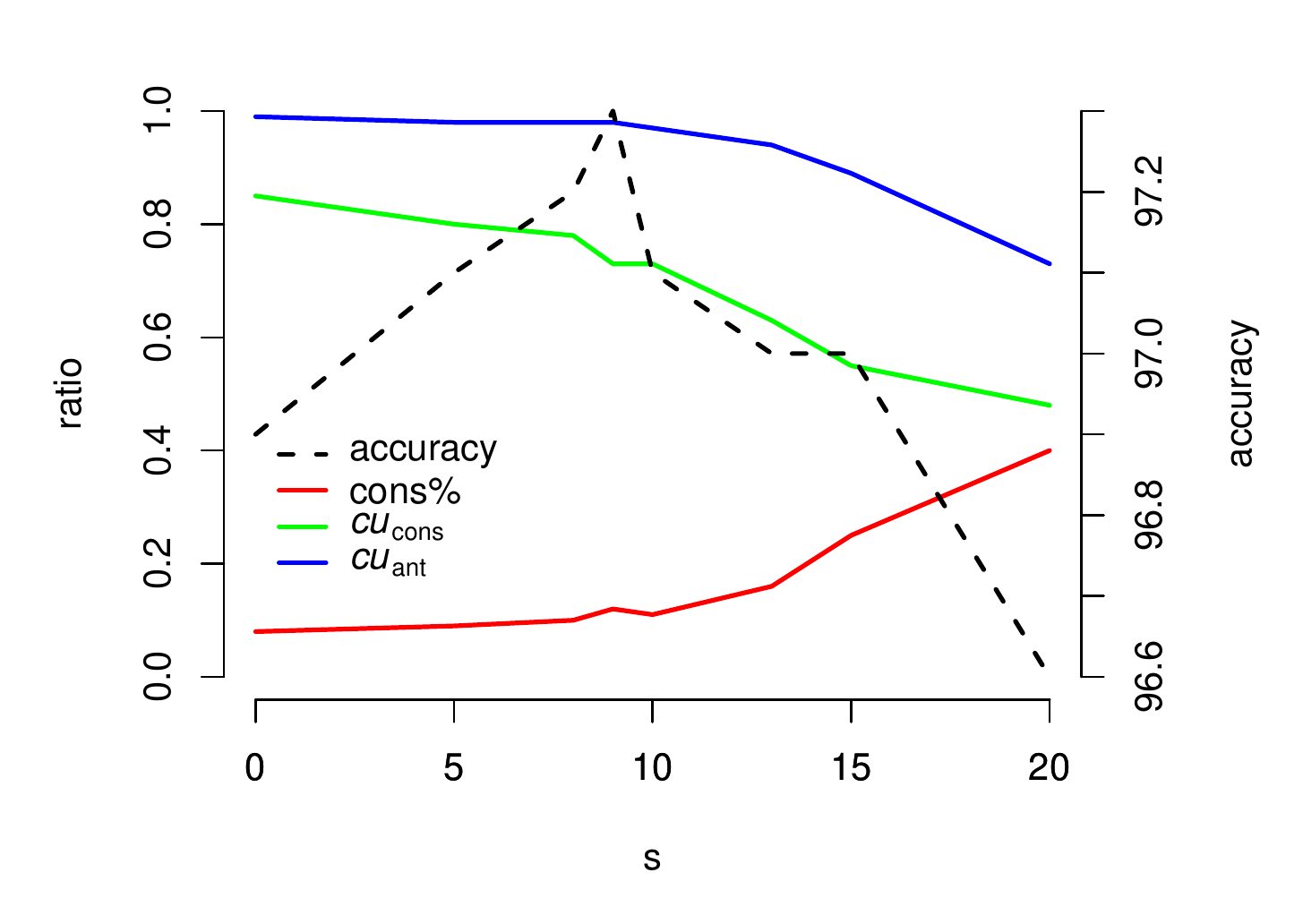}
\caption{The results using the Reichenbach-sigmoidal implication $\sigma_{I_{RC}}$ for various values of $s$.}
\label{fig:mnist_s_experiments}
\end{figure}
In Figure \ref{fig:mnist_s_experiments} we find the results when we experiment with the parameter $s$. Note that when $s$ approaches 0 the Reichenbach-sigmoidal implication is $I_{RC}$. The value of 9 gives the best results, with 97.3\% accuracy. Interestingly enough, there seem to be clear trends in the values of $\mpratio$,\ $\mpupdateratio$ and $\mtupdateratio$. Increasing $s$ seems to increase  $\mpratio$. This is because the antecedent derivative around the corner $a=0,\ c=0$ will be low, as argued before. When $s$ increases, the corners will be more smoothed out. 
Furthermore, both $\mpupdateratio$ and $\mtupdateratio$ decrease when $s$ increases. This could be because around the corners the derivatives become small. Updates in the corner will likely be correct as the model is already confident about those. For a higher value of $s$, most of the gradient magnitude is at instances on which the model is less confident. Regardless, the best parameter value clearly is not the one for which the values of $\mpupdateratio$ and $\mtupdateratio$ are highest, namely the Reichenbach implication itself. 

\subsubsection{Influence of Individual Formulas}
\label{sec:mnist_formulas_experiments}
\begin{table}
\begin{tabular}{l....}
\hline
Formulas                         & \mc{Accuracy} & \mc{$\mpratio$} & \mc{$\mpupdateratio$} & \mc{$\mtupdateratio$}               \\
\hline
(1) \& (2) & \bft{97.1}    & 0.05            & 0.54                  & \bft{0.99} \\
(2) \& (3)  & 95.9          & 0.12            & \bft{0.75}            & 0.95                           \\
(1) \& (3)  & 96.3          & 0.15            & 0.52                  & 0.98                           \\
(1)                     & 95.6          & 0.05            & 0.59                  & \bft{1.00} \\
(2)                     & 95.2          & 0.03            & \bft{0.78}            & 0.99                           \\
(3)                      & \bft{95.8}    & 0.19            & 0.64                  & 0.95               \\
\hline
\end{tabular}
\caption{The results using $\sigma_{I_{RC}}$ for the implication with $s=9$, leaving some formulas out. The numbers indicated the formulas that are present during training.  }
\label{table:mnist_formula_experiments}
\end{table}
Finally, we compare what the influence of the different formulas are in Table \ref{table:mnist_formula_experiments}. Removing the reflexivity formula (3) does not largely impact the performance. The biggest drop in performance is by removing formula (1) that defines the $\pred{same}$ predicate. Using only formula (1) gets slightly better performance than only using formula (2), despite the fact that no positive labeled examples can be found using formula (1) as the predicates $\pred{zero}$ to $\pred{nine}$ are not in its consequent. Since 95\% of the derivatives are with respect to the negated antecedent, this formula contributes by finding additional counterexamples. Furthermore, improving the accuracy of the $\pred{same}$ predicate improves the accuracy on digit recognition: Just using the reflexivity formula (3) has the highest accuracy when used individually, even though it does not use the digit predicates.

\subsubsection{Analysis}


\dualfigure{\imgE cons_cucons.pdf}{\imgE cons_cuant.pdf}{Left: Plot of $\mpratio$ to $\mpupdateratio$. Right: Plot of $\mpratio$ to $\mtupdateratio$. Red dots represent runs using $\sigma_{I_{RC}}$.}{fig:analyze_measures}

We plot the experimental values of $\mpratio$ to the values of $\mpupdateratio$ and $\mtupdateratio$ in Figure \ref{fig:analyze_measures}. For both, there seems to be a negative correlation. Apparently, if the ratio of derivatives with respect to the consequent becomes larger, then this decreases the correctness of the updates. In Section \ref{sec:mnist_rcsigmoidal} we argued, when experimenting with the value of $s$, that this could be because for lower values of $\mpratio$, a smaller portion of the reasoning happens in the `safe' corners around $a=0,\ c=0$ and $a=1,\ c=1$, and more for cases that the agent is less certain about. As all S-implications have strong derivatives at both these corners (Proposition \ref{prop:diff_left_neutral}), this phenomenon is likely present in other S-implications.

Although \dfl significantly improves on the supervised baseline and is thus suited for semi-supervised learning, it is currently not competitive with state-of-the-art methods like Ladder Networks \pcite{rasmus2015semi} which has an accuracy of 98.9\% for 100 labeled pictures and 99.2\% for 1000.

\section{Related Work}
\label{chapter:related_work}
\dfuzz falls into the discipline of Statistical Relational Learning \pcite{getoor2007}, which concerns models that can reason under uncertainty and learn relational structures like graphs. 
Special cases of \dfl have been researched in several papers under different names. Real Logic \pcite{Serafini2016} implements function symbols and uses a model called Logic Tensor Networks to interpret predicates. It uses S-implications. Real Logic is applied to weakly supervised learning on Semantic Image Interpretation \pcite{Donadello2017,donadello2019compensating} and transfer learning in Reinforcement Learning \pcite{badreddine2019injecting}.
Semantic-based regularization (SBR) \pcite{diligenti2017b} applies \dfl to kernel machines. They use R-implications, like \tcite{marra2019learning} 
which simplifies the satisfiability computation and finds generalizations of common loss functions can be found. In \tcite{marra2018}, which employs the Goguen implication, \dfl is applied to image generation. By using function symbols that represent generator neural networks, they create constraints that are used to create a semantic description of an image generation problem. The Reichenbach implication is used in \tcite{Rocktaschel2015b} for relation extraction by using an efficient matrix embedding of the rules. 

The regularization technique used in \tcite{Demeester2016} is equivalent to the \luk\ implication. Instead of using existing data, it finds a loss function which does not iterate over objects, yet can guarantee that the rules hold. A promising approach is using adversarial sets \pcite{minervini2017}, which is a set of objects from the domain that do not satisfy the knowledge base, which are probably the most informative objects. Adversarial sets are applied to natural language interpretation in \pcite{minervini2018}. Both papers use the \luk\ implication. 

Some approaches use probabilistic logics instead of fuzzy logics and interpret predicates probabilistically. DeepProbLog \pcite{DBLP:conf/nips/2018} and Semantic Loss  \pcite{pmlr-v80-xu18h} are probabilistic logic programming languages with neural predicates that compute the probabilities of ground atoms. They support automatic differentiation which can be used to back-propagate from the loss at a query predicate to the deep learning models that implement the neural predicates, similar to \dfl.

\label{chapter:conclusions}

\section{Conclusion}
We analyzed fuzzy implications in \dfuzz in order to understand how reasoning using implications behaves in a differentiable setting.
We have found substantial differences between the properties of a large number of fuzzy implications, and showed that many of them, including some of the most popular implications, are highly unsuitable for use in a differentiable learning setting. 

The Reichenbach implication has derivatives that are intuitive and that correspond to inference rules from classical logic. 
The \luk\ implication is the best R-implication in our experiments. The Gödel and Goguen implications, on the other hand, were much less successful, performing worse than the supervised baseline. The newly introduced Reichenbach-sigmoidal implication performs best on the MNIST experiments. 
The spread of sigmoidal implications can be tweaked to decrease the imbalance of the derivatives with respect to the negated antecedent and consequent. 

We noted an interesting imbalance between derivatives with respect to the negated antecedent and the consequent of the implication. Because the modus tollens case is much more common, we conclude that a large part of the useful inferences on the MNIST experiments are made by decreasing the antecedent, or by `modus tollens reasoning'. Furthermore, we found that derivatives with respect to the consequent often increase the truth value of something that is false as the consequent is false in the majority of times. Therefore, we argue that `modus tollens reasoning' should be embraced in future research. 


\bibliography{references}

\begin{thebibliography}{}

\bibitem[\protect\citeauthoryear{Badreddine and
  Spranger}{2019}]{badreddine2019injecting}
Badreddine, S., and Spranger, M.
\newblock 2019.
\newblock {Injecting Prior Knowledge for Transfer Learning into Reinforcement
  Learning Algorithms using Logic Tensor Networks}.
\newblock {\em arXiv preprint arXiv:1906.06576}.

\bibitem[\protect\citeauthoryear{Besold \bgroup et al\mbox.\egroup
  }{2017}]{Besold2017a}
Besold, T.~R.; Garcez, A.~d.; Bader, S.; Bowman, H.; Domingos, P.; Hitzler, P.;
  Kuehnberger, K.-U.; Lamb, L.~C.; Lowd, D.; Lima, P. M.~V.; de~Penning, L.;
  Pinkas, G.; Poon, H.; and Zaverucha, G.
\newblock 2017.
\newblock {Neural-Symbolic Learning and Reasoning: A Survey and
  Interpretation}.
\newblock {\em arXiv preprint arXiv:1711.03902}.

\bibitem[\protect\citeauthoryear{Brock, Donahue, and
  Simonyan}{2018}]{brock2018large}
Brock, A.; Donahue, J.; and Simonyan, K.
\newblock 2018.
\newblock {Large scale gan training for high fidelity natural image synthesis}.
\newblock {\em arXiv preprint arXiv:1809.11096}.

\bibitem[\protect\citeauthoryear{Demeester, Rockt{\"{a}}schel, and
  Riedel}{2016}]{Demeester2016}
Demeester, T.; Rockt{\"{a}}schel, T.; and Riedel, S.
\newblock 2016.
\newblock {Lifted Rule Injection for Relation Embeddings}.
\newblock In {\em Proceedings of the 2016 Conference on Empirical Methods in
  Natural Language Processing},  1389--1399.
\newblock Association for Computational Linguistics.

\bibitem[\protect\citeauthoryear{Diligenti, Gori, and
  Sacca}{2017}]{diligenti2017b}
Diligenti, M.; Gori, M.; and Sacca, C.
\newblock 2017.
\newblock {Semantic-based regularization for learning and inference}.
\newblock {\em Artificial Intelligence} 244:143--165.

\bibitem[\protect\citeauthoryear{Diligenti, Roychowdhury, and
  Gori}{2017}]{diligenti2017}
Diligenti, M.; Roychowdhury, S.; and Gori, M.
\newblock 2017.
\newblock {Integrating Prior Knowledge into Deep Learning}.
\newblock In {\em Machine Learning and Applications (ICMLA), 2017 16th IEEE
  International Conference on},  920--923.
\newblock IEEE.

\bibitem[\protect\citeauthoryear{Donadello and
  Serafini}{2019}]{donadello2019compensating}
Donadello, I., and Serafini, L.
\newblock 2019.
\newblock {Compensating Supervision Incompleteness with Prior Knowledge in
  Semantic Image Interpretation}.
\newblock In {\em 2019 International Joint Conference on Neural Networks
  (IJCNN)},  1--8.
\newblock IEEE.

\bibitem[\protect\citeauthoryear{Donadello, Serafini, and
  Garcez}{2017}]{Donadello2017}
Donadello, I.; Serafini, L.; and Garcez, A.~d.
\newblock 2017.
\newblock {Logic Tensor Networks for Semantic Image Interpretation}.
\newblock In {\em IJCAI International Joint Conference on Artificial
  Intelligence},  1596--1602.

\bibitem[\protect\citeauthoryear{Evans and Grefenstette}{2018}]{Evans2018}
Evans, R., and Grefenstette, E.
\newblock 2018.
\newblock {Learning explanatory rules from noisy data}.
\newblock {\em Journal of Artificial Intelligence Research} 61:65--170.

\bibitem[\protect\citeauthoryear{Garcez, Broda, and
  Gabbay}{2012}]{garcez2012neural}
Garcez, A. S.~d.; Broda, K.~B.; and Gabbay, D.~M.
\newblock 2012.
\newblock {\em {Neural-symbolic learning systems: foundations and
  applications}}.
\newblock Springer Science {\&} Business Media.

\bibitem[\protect\citeauthoryear{Garnelo, Arulkumaran, and
  Shanahan}{2016}]{garnelo2016towards}
Garnelo, M.; Arulkumaran, K.; and Shanahan, M.
\newblock 2016.
\newblock {Towards deep symbolic reinforcement learning}.
\newblock {\em arXiv preprint arXiv:1609.05518}.

\bibitem[\protect\citeauthoryear{Getoor and Taskar}{2007}]{getoor2007}
Getoor, L., and Taskar, B.
\newblock 2007.
\newblock {\em {Introduction to statistical relational learning}}, volume~1.
\newblock MIT press Cambridge.

\bibitem[\protect\citeauthoryear{Goodfellow \bgroup et al\mbox.\egroup
  }{2016}]{goodfellow2016deep}
Goodfellow, I.; Bengio, Y.; Courville, A.; and Bengio, Y.
\newblock 2016.
\newblock {\em {Deep learning}}, volume~1.
\newblock MIT press Cambridge.

\bibitem[\protect\citeauthoryear{Guo \bgroup et al\mbox.\egroup
  }{2016}]{guo2016}
Guo, S.; Wang, Q.; Wang, L.; Wang, B.; and Guo, L.
\newblock 2016.
\newblock {Jointly embedding knowledge graphs and logical rules}.
\newblock In {\em Proceedings of the 2016 Conference on Empirical Methods in
  Natural Language Processing},  192--202.

\bibitem[\protect\citeauthoryear{Harnad}{1990}]{harnad1990symbol}
Harnad, S.
\newblock 1990.
\newblock {The symbol grounding problem}.
\newblock {\em Physica D: Nonlinear Phenomena} 42(1-3):335--346.

\bibitem[\protect\citeauthoryear{Hempel}{1945}]{Hempel1945}
Hempel, C.~G.
\newblock 1945.
\newblock {Studies in the Logic of Confirmation (II.)}.
\newblock {\em Mind} 54(214):97--121.

\bibitem[\protect\citeauthoryear{Hu \bgroup et al\mbox.\egroup
  }{2016}]{P16-1228}
Hu, Z.; Ma, X.; Liu, Z.; Hovy, E.; and Xing, E.
\newblock 2016.
\newblock {Harnessing Deep Neural Networks with Logic Rules}.
\newblock In {\em Proceedings of the 54th Annual Meeting of the Association for
  Computational Linguistics (Volume 1: Long Papers)},  2410--2420.
\newblock Association for Computational Linguistics.

\bibitem[\protect\citeauthoryear{Japkowicz and
  Stephen}{2002}]{japkowicz2002class}
Japkowicz, N., and Stephen, S.
\newblock 2002.
\newblock {The class imbalance problem: A systematic study}.
\newblock {\em Intelligent data analysis} 6(5):429--449.

\bibitem[\protect\citeauthoryear{Jayaram and Baczynski}{2008}]{Jayaram2008}
Jayaram, B., and Baczynski, M.
\newblock 2008.
\newblock {\em {Fuzzy Implications}}, volume 231.
\newblock Springer, Berlin, Heidelberg.

\bibitem[\protect\citeauthoryear{Klir and Yuan}{1995}]{klir1995fuzzy}
Klir, G., and Yuan, B.
\newblock 1995.
\newblock {\em {Fuzzy sets and fuzzy logic}}, volume~4.
\newblock Prentice hall New Jersey.

\bibitem[\protect\citeauthoryear{LeCun and
  Cortes}{2010}]{lecun-mnisthandwrittendigit-2010}
LeCun, Y., and Cortes, C.
\newblock 2010.
\newblock {MNIST handwritten digit database}.

\bibitem[\protect\citeauthoryear{Manhaeve \bgroup et al\mbox.\egroup
  }{2018}]{DBLP:conf/nips/2018}
Manhaeve, R.; Duman{\v{c}}i{\'{c}}, S.; Kimmig, A.; Demeester, T.; and
  De~Raedt, L.
\newblock 2018.
\newblock {DeepProbLog: Neural Probabilistic Logic Programming}.
\newblock In Bengio, S.; Wallach, H.~M.; Larochelle, H.; Grauman, K.;
  Cesa-Bianchi, N.; and Garnett, R., eds., {\em Advances in Neural Information
  Processing Systems 31: Annual Conference on Neural Information Processing
  Systems 2018, NeurIPS 2018, 3-8 December 2018, Montr{\'{e}}al, Canada}.

\bibitem[\protect\citeauthoryear{Marcus}{2018}]{marcus2018deep}
Marcus, G.
\newblock 2018.
\newblock {Deep learning: A critical appraisal}.
\newblock {\em arXiv preprint arXiv:1801.00631}.

\bibitem[\protect\citeauthoryear{Marra \bgroup et al\mbox.\egroup
  }{2018}]{marra2018}
Marra, G.; Giannini, F.; Diligenti, M.; and Gori, M.
\newblock 2018.
\newblock {Constraint-Based Visual Generation}.
\newblock {\em arXiv preprint arXiv:1807.09202}.

\bibitem[\protect\citeauthoryear{Marra \bgroup et al\mbox.\egroup
  }{2019}]{marra2019learning}
Marra, G.; Giannini, F.; Diligenti, M.; Maggini, M.; and Gori, M.
\newblock 2019.
\newblock {Learning and T-Norms Theory}.
\newblock {\em arXiv preprint arXiv:1907.11468}.

\bibitem[\protect\citeauthoryear{Minervini and Riedel}{2018}]{minervini2018}
Minervini, P., and Riedel, S.
\newblock 2018.
\newblock {Adversarially Regularising Neural NLI Models to Integrate Logical
  Background Knowledge}.
\newblock In {\em Proceedings of the 22nd Conference on Computational Natural
  Language Learning},  65--74.

\bibitem[\protect\citeauthoryear{Minervini \bgroup et al\mbox.\egroup
  }{2017}]{minervini2017}
Minervini, P.; Demeester, T.; Rockt{\"{a}}schel, T.; and Riedel, S.
\newblock 2017.
\newblock {Adversarial sets for regularising neural link predictors}.
\newblock In {\em Uncertainty in Artificial Intelligence-Proceedings of the
  33rd Conference, UAI 2017}.

\bibitem[\protect\citeauthoryear{Pearl}{2018}]{pearl2018theoretical}
Pearl, J.
\newblock 2018.
\newblock {Theoretical Impediments to Machine Learning With Seven Sparks from
  the Causal Revolution}.
\newblock In {\em Proceedings of the Eleventh ACM International Conference on
  Web Search and Data Mining}, ~3.
\newblock ACM.

\bibitem[\protect\citeauthoryear{Radford \bgroup et al\mbox.\egroup
  }{2019}]{radford2019language}
Radford, A.; Wu, J.; Child, R.; Luan, D.; Amodei, D.; and Sutskever, I.
\newblock 2019.
\newblock {Language Models are Unsupervised Multitask Learners}.

\bibitem[\protect\citeauthoryear{Rasmus \bgroup et al\mbox.\egroup
  }{2015}]{rasmus2015semi}
Rasmus, A.; Berglund, M.; Honkala, M.; Valpola, H.; and Raiko, T.
\newblock 2015.
\newblock {Semi-supervised learning with ladder networks}.
\newblock In {\em Advances in Neural Information Processing Systems},
  3546--3554.

\bibitem[\protect\citeauthoryear{Rockt{\"{a}}schel, Singh, and
  Riedel}{2015}]{Rocktaschel2015b}
Rockt{\"{a}}schel, T.; Singh, S.; and Riedel, S.
\newblock 2015.
\newblock {Injecting Logical Background Knowledge into Embeddings for Relation
  Extraction}.
\newblock In {\em Proceedings of the 2015 Conference of the North American
  Chapter of the Association for Computational Linguistics: Human Language
  Technologies},  1119--1129.

\bibitem[\protect\citeauthoryear{Serafini and Garcez}{2016}]{Serafini2016}
Serafini, L., and Garcez, A.~D.
\newblock 2016.
\newblock {Logic tensor networks: Deep learning and logical reasoning from data
  and knowledge}.
\newblock {\em CEUR Workshop Proceedings} 1768.

\bibitem[\protect\citeauthoryear{Silver \bgroup et al\mbox.\egroup
  }{2017}]{silver2017mastering}
Silver, D.; Schrittwieser, J.; Simonyan, K.; Antonoglou, I.; Huang, A.; Guez,
  A.; Hubert, T.; Baker, L.; Lai, M.; Bolton, A.; and {others}.
\newblock 2017.
\newblock {Mastering the game of Go without human knowledge}.
\newblock {\em Nature} 550(7676):354.

\bibitem[\protect\citeauthoryear{Socher \bgroup et al\mbox.\egroup
  }{2013}]{Socher2013}
Socher, R.; Chen, D.; Manning, C.~D.; and Ng, A.~Y.
\newblock 2013.
\newblock {Reasoning With Neural Tensor Networks for Knowledge Base
  Completion}.
\newblock {\em Proc.{\textbackslash} of NIPS'13}  1--10.

\bibitem[\protect\citeauthoryear{van Krieken, Acar, and van
  Harmelen}{2019}]{vankrieken2019ravens}
van Krieken, E.; Acar, E.; and van Harmelen, F.
\newblock 2019.
\newblock {Semi-Supervised Learning using Differentiable Reasoning}.
\newblock {\em IFCoLog Journal of Logic and its Applications} 6(4):633--653.

\bibitem[\protect\citeauthoryear{Vranas}{2004}]{Vranas2004}
Vranas, P.~B.
\newblock 2004.
\newblock {Hempel's raven paradox: A lacuna in the standard Bayesian solution}.
\newblock {\em British Journal for the Philosophy of Science} 55(3):545--560.

\bibitem[\protect\citeauthoryear{Xu \bgroup et al\mbox.\egroup
  }{2018}]{pmlr-v80-xu18h}
Xu, J.; Zhang, Z.; Friedman, T.; Liang, Y.; and den Broeck, G.
\newblock 2018.
\newblock {A Semantic Loss Function for Deep Learning with Symbolic Knowledge}.
\newblock In Dy, J., and Krause, A., eds., {\em Proceedings of the 35th
  International Conference on Machine Learning}, volume~80 of {\em Proceedings
  of Machine Learning Research},  5502--5511.
\newblock Stockholmsm{\"{a}}ssan, Stockholm Sweden: PMLR.

\bibitem[\protect\citeauthoryear{Zhou}{2017}]{zhou2017}
Zhou, Z.-H.
\newblock 2017.
\newblock {A brief introduction to weakly supervised learning}.
\newblock {\em National Science Review} 5(1):44--53.

\end{thebibliography}
\bibliographystyle{named}


\end{document}